\documentclass{article}

\usepackage{mdframed}
\usepackage{bbding}
\usepackage{footnote}
\makesavenoteenv{tabular}
\makesavenoteenv{table}

\let\floor\undefined
\let\ceil\undefined
\usepackage{upgreek}
\usepackage{amsmath}
\usepackage{amsthm}

\usepackage{amssymb}
\usepackage{mathrsfs}
\usepackage{bbm}
\usepackage{bm}
\usepackage{multirow}
\usepackage{enumitem}
\usepackage{xcolor}
\usepackage{mathtools}
\usepackage{url}            
\usepackage{booktabs}       
\usepackage{microtype}      
\usepackage{upgreek}
\mathtoolsset{showonlyrefs}

\newtheorem{theorem}{Theorem}
\newmdtheoremenv{ftheorem}{Theorem}

\newtheorem{lemma}[theorem]{Lemma}
\newtheorem{definition}[theorem]{Definition}

\newcommand{\X}{X}

\newcommand{\p}{\bm{p}}

\newcommand{\cL}{\mathcal{L}}
\newcommand{\eff}{\varphi}

\usepackage[mathscr]{euscript}
  
\newcommand*{\rmd}{\mathrm{d}}
\renewcommand{\P}{ {P}}
\newcommand{\Q}{ {Q}}

\newcommand{\cM}{\mathcal{M}}
\newcommand{\cQ}{ \mathcal{Q}}

\newcommand{\cH}{\mathcal{H}}

\newcommand{\wtilde}{\widetilde}

\newcommand{\loss}{\ell}
\DeclareBoldMathCommand{\vloss}{\loss}
\DeclareBoldMathCommand{\grad}{g}
\DeclareBoldMathCommand{\fakegrad}{\mathring{\bm{g}}}
\DeclareBoldMathCommand{\e}{e}
\DeclareBoldMathCommand{\p}{p}
\DeclareBoldMathCommand{\u}{u}
\DeclareBoldMathCommand{\w}{w}
\DeclareBoldMathCommand{\x}{x}
\DeclareBoldMathCommand{\l}{l}
\DeclareBoldMathCommand{\vzero}{0}
\let\top\intercal

\newcommand{\R}{\ensuremath{\textsc R}}

\newcommand{\C}{\ensuremath{\textsc C}}

\newcommand{\cvar}{\ensuremath{\textsc {CVaR}}}

\newcommand{\er}{\ensuremath{\textsc {ER}}}

\newcommand{\Var}{\textsc{Var}}

\newcommand{\reals}{\mathbb{R}}
\DeclareMathOperator{\KL}{KL}
\DeclareMathOperator{\E}{\ensuremath{\textsc E}}

\renewcommand{\x}{\bm{x}}

\newcommand{\cX}{\mathcal{X}}

\newcommand{\cG}{\mathcal{G}}

\DeclareMathOperator{\D}{\op{D}}

\DeclareMathOperator*{\argmin}{argmin}
\DeclareMathOperator*{\argmax}{argmax}

\newcommand*{\what}[1]{\widehat{#1}}
\newcommand*{\floor}[1]{\left\lfloor{#1}\right\rfloor}

\newcommand*{\ceil}[1]{\lceil{#1}\rceil}

\newcommand*{\op}[1]{\operatorname{#1}}

\newmdtheoremenv{condition}{Condition}

\newcommand{\vertiii}[1]{{\left\vert\kern-0.25ex\left\vert\kern-0.25ex\left\vert #1 
    \right\vert\kern-0.25ex\right\vert\kern-0.25ex\right\vert}}

\makeatletter
\newcommand*\bcdot{\mathpalette\bigcdot@{.5}}
\newcommand*\bigcdot@[2]{\mathbin{\vcenter{\hbox{\scalebox{#2}{$\m@th#1\bullet$}}}}}

\newlength\myindent
\newcommand{\rhohat}{\what{\uprho}}
\newcommand{\comp}{\mathscr{K}_n}
\newcommand{\cO}{\mathcal{O}}

\usepackage[utf8]{inputenc} 
\usepackage[T1]{fontenc}    
\usepackage[colorlinks = true, citecolor = blue,linkcolor = orange]{hyperref}       
\usepackage{url}            
\usepackage{booktabs}       
\usepackage{amsfonts}       
\usepackage{nicefrac}       
\usepackage{microtype}      

\usepackage{subfigure}

\usepackage{MnSymbol}
\let\oldproofname=\proofname
\renewcommand{\proofname}{\rm\bf{\oldproofname}\nopunct}

\PassOptionsToPackage{sort}{natbib}
\usepackage[preprint]{neurips_2020}

\title{PAC-Bayesian Bound for the Conditional Value at Risk} 

\author{
Zakaria Mhammedi\\
The Australian National University and Data61\\
\texttt{zak.mhammedi@anu.edu.au}
\And
Benjamin Guedj\\
Inria and University College London\\
\texttt{benjamin.guedj@inria.fr}
\And
Robert C.\ Williamson\\
\texttt{bobwilliamsonoz@icloud.com}
}

\begin{document}

\maketitle

\begin{abstract}
Conditional Value at Risk ($\cvar$) is a family of ``coherent risk measures'' which generalize the traditional mathematical 
expectation. 
Widely used in mathematical finance, it is garnering increasing interest in machine learning, e.g., as an alternate approach to 
regularization, and as a means for ensuring fairness.  
This paper presents a generalization bound for learning algorithms that minimize the $\cvar$ of the empirical loss. 
The bound is of PAC-Bayesian type and is guaranteed to be small when the empirical $\cvar$ is small. 
We achieve this by reducing the problem of estimating $\cvar$ to that of merely estimating an expectation. This then enables us, as a by-product, to obtain concentration inequalities for $\cvar$ even when the random variable in question is unbounded.
\end{abstract}

\section{Introduction}
The goal in statistical learning is to learn hypotheses that generalize well, which is typically formalized by seeking to minimize the \emph{expected} risk  associated with a given loss function. In general, a loss function is a map $\ell\colon  \cH\times  \mathcal{X} \rightarrow  \reals_{\geq 0}$, where $\mathcal{X}$ is a feature space and $\cH$ is an hypotheses space. In this case, the \emph{expected risk} associated with a given hypothesis $h \in \cH$ is given by $\R[\ell(h,X)] \coloneqq \E[\ell(h,X)]$. Since the data-generating distribution is typically unknown, the expected risk is approximated using observed i.i.d.~samples $X_1,\dots, X_n$ of $X$, and an hypothesis is then chosen to minimize the empirical risk $\what \R[\ell(h,X)] \coloneqq \sum_{i=1}^n\ell(h, X_i)/n$. When choosing an hypothesis $\hat h$ based on the empirical risk $\what \R$, one would like to know how close $\what \R[\ell(\hat h, X)]$ is to the actual risk $\R[\ell(\hat h, X)]$; only then can one infer something about the \emph{generalization} property of the learned hypothesis $\hat h$.  

Generalization bounds---in which the expected risk is bounded in terms of its empirical version up to some error---are at the heart of many machine learning problems. The main techniques leading to such bounds comprise uniform converges arguments (often involving the Rademacher complexity of the set $\cH$), algorithmic stability arguments (see e.g. \citep{bousquet2002stability} and more recently the work from \citep{DBLP:conf/alt/Abou-MoustafaS19,DBLP:journals/corr/abs-1910-07833,celisse2016stability}), and the PAC-Bayesian analysis for non-degenerate randomized estimators \citep{mcallester2003}. Behind these techniques lies \emph{concentration inequalities}, such as Chernoff's inequality (for the PAC-Bayesian analysis) and McDiarmid's inequality (for algorithmic stability and the uniform convergence analysis), which control the deviation between population and empirical averages \citep[see][among others]{boucheron2003concentration,boucheron2013concentration,mcdiarmid1998concentration}.

Standard concentration inequalities are well suited for learning problems where the goal is to minimize the expected risk $\E[\ell(h,X)]$. However, the expected risk---the mean performance of an algorithm---might fail to capture the underlying phenomenon of interest. For example, when dealing with medical (responsitivity to a specific drug with grave side effects, etc.), environmental (such as pollution, exposure to toxic compounds, etc.), or sensitive engineering tasks (trajectory evaluation for autonomous vehicles, etc.), the mean performance is not necessarily the best objective to optimize as it will cover potentially disastrous mistakes (\emph{e.g.}, a few extra centimeters when crossing another vehicle, a slightly too large dose of a lethal compound, etc.) while possibly improving on average.
There is thus a growing interest to work with alternative measures of risk (other than the expectation) for which standard concentration inequalities do not apply directly.
Of special interest are \emph{coherent risk measures} \citep{artzner1999coherent} which possess properties that make them desirable in mathematical finance and portfolio optimization \citep{allais1953comportement, ellsberg1961risk, rockafellar2000optimization}, with a focus on optimizing for the worst outcomes rather than on average. Coherent risk measures have also been recently connected to fairness, and appear as a promising framework to control the fairness of an algorithm's solution \citep{williamson2019fairness}.

A popular coherent risk measure is the Conditional Value at Risk ($\cvar$; see \citealp{Pflug2000}); for $\alpha \in(0,1)$ and random variable $Z$, $\cvar_{\alpha}[Z]$ measures the expectation of $Z$ conditioned on the event that $Z$ is greater than its $(1-\alpha)$-th quantile. $\cvar$ has been shown to underlie the classical SVM \citep{takeda2008nu}, and has in general attracted a large interest in machine learning over the past two decades 
\citep{huo2017risk, bhat2019concentration,williamson2019fairness,Chen2009,Chow2014,prashanth2013actor,tamar2015optimizing,pinto2017robust,morimura2010nonparametric,Takeda2009}.

Various concentration inequalities have been derived for $\cvar_{\alpha}[Z]$, under different assumptions on $Z$, which bound the difference between $\cvar_{\alpha}[Z]$ and its standard estimator $\what \cvar_{\alpha}[Z]$ with high probability \citep{brown2007, wang2010, prashanth2013actor, kolla2019concentration, bhat2019concentration}. However, none of these works extend their results to the statistical learning setting where the goal is to learn an hypothesis from data to minimize the conditional value at risk. In this paper, we fill this gap by presenting a sharp PAC-Bayesian generalization bound when the objective is to minimize the conditional value at risk. 

\paragraph{Related Works.} Deviation bounds for $\cvar$ were first presented by \citet{brown2007}. However, their approach only applies to bounded continuous random variables, and their lower deviation bound has a sub-optimal dependence on the level $\alpha$. \citet{wang2010} later refined their analysis to recover the ``correct'' dependence in $\alpha$, albeit their technique still requires a two-sided bound on the random variable $Z$. \citet{thomas2019} derived new concentration inequalities for $\cvar$ with a very sharp empirical performance, even though the dependence on $\alpha$ in their bound is sub-optimal. Further, they only require a one-sided bound on $Z$, without a continuity assumption. 

\citet{kolla2019concentration} were the first to provide concentration bounds for $\cvar$ when the random variable $Z$ is unbounded, but is either sub-Gaussian or sub-exponential. \citet{bhat2019concentration} used a bound on the Wasserstein distance between true and empirical cumulative distribution functions to substantially tighten the bounds of \citet{kolla2019concentration} when $Z$ has finite exponential or $k$th-order moments; they also apply their results to other coherent risk measures. However, when instantiated with bounded random variables, their concentration inequalities have sub-optimal dependence in $\alpha$.

On the statistical learning side, \cite{duchi2018learning} present generalization bounds for a class of coherent risk measures that technically includes $\cvar$. However, their bounds are based on uniform convergence arguments which lead to looser bounds compared with ours. 

\paragraph{Contributions.}
Our main contribution is a \emph{PAC-Bayesian generalization bound for the conditional value at risk}, where we bound the difference $\cvar_{\alpha}[Z] -\what\cvar_{\alpha}[Z]$, for $\alpha \in(0,1)$, by a term of order $\sqrt{{\what\cvar_{\alpha}[Z] \cdot \comp}/({n\alpha})},$ with $\comp$ representing a complexity term which depends on $\cH$. Due to the presence of $\what\cvar_{\alpha}[Z]$ inside the square-root, our generalization bound has the desirable property that it becomes small whenever the empirical conditional value at risk is small. For the standard expected risk, only state-of-the-art PAC-Bayesian bounds share this property (see \emph{e.g.} \cite{langford2003pac,catoni2007pac,maurer2004note} or more recently in \cite{DBLP:conf/nips/TolstikhinS13,mhammedi2019pac}). We refer to \citep{guedj2019primer} for a recent survey on PAC-Bayes.

As a by-product of our analysis, we derive a new way of obtaining \emph{concentration bounds for the conditional value at risk} by reducing the problem to estimating expectations using empirical means. This reduction then makes it easy to obtain concentration bounds for $\cvar_{\alpha}[Z]$ even when the random variable $Z$ is unbounded ($Z$ may be sub-Gaussian or sub-exponential). Our bounds have explicit constants and are sharper than existing ones due to \cite{kolla2019concentration,bhat2019concentration} which deal with the unbounded case.

\paragraph{Outline.} In Section~\ref{sec:notation}, we define the conditional value at risk along with its standard estimator. In Section~\ref{sec:def}, we recall the statistical learning setting and present our PAC-Bayesian bound for $\cvar$. The proof of our main bound is in Section~\ref{sec:PACBayes}. In Section~\ref{sec:concentration}, we present a new way of deriving concentration bounds for $\cvar$ which stems from our analysis in Section~\ref{sec:PACBayes}.  Section~\ref{sec:entropies} concludes and suggests  future  directions.

\section{Preliminaries}\label{sec:notation}
Let $(\Omega, \mathcal{F}, \P)$ be a probability space. For $p\in \mathbb{N}$, we denote by $\mathcal{L}^p(\Omega) \coloneqq \mathcal{L}^p(\Omega, \mathcal{F},\P)$ the space of $p$-integrable functions, and we let $\mathcal{M}_{\P}(\Omega)$ be the set of probability measures on $\Omega$ which are absolutely continuous with respect to $\P$.  
We reserve the notation $\E$ for the expectation under the reference measure $\P$, although we sometimes write $\E_{\P}$ for clarity. For random variables $Z_1, \dots,Z_n$, we denote $\what P_n\coloneqq \sum_{i=1}^n \updelta_{Z_i}/n$ the empirical distribution, and we let $Z_{1:n}\coloneqq (Z_1,\dots,Z_n)$. Furthermore, we let $\uppi \coloneqq (1,\dots,1)^\top/n \in \reals^n$  be the uniform distribution on the simplex. Finally, we use the notation $\wtilde{\cO}$ to hide log-factors in the sample size $n$.
\paragraph{Coherent Risk Measures (CRM).} A CRM \citep{artzner1999coherent} is a functional $\R\colon \mathcal{L}^1(\Omega) \rightarrow \reals \cup \{+\infty\}$ that is simultaneously, positive homogeneous, monotonic, translation equivariant, and sub-additive\footnote{These are precisely the properties which make coherent risk measures excellent candidates in some machine learning applications (see e.g.~\citep{williamson2019fairness} for an application to fairness)} (see Appendix \ref{app:gen} for a formal definition). For $\alpha \in(0,1)$ and a real random variable $Z\in \cL_1(\Omega)$, the conditional value at risk $\cvar_{\alpha}[Z]$ is a CRM and is defined as the mean of the random variable $Z$ conditioned on the event that $Z$ is greater than its $(1-\alpha)$-th quantile\footnote{We use the convention in \cite{brown2007,wang2010, prashanth2013actor}.}. 
This is equivalent to the following expression, which is more convenient for our analysis: 
\begin{align}
\cvar_{\alpha}[Z] = \C_{\alpha}[Z] \coloneqq \inf_{\mu \in \reals} \left\{ \mu + \frac{\E[ Z-\mu]_+}{\alpha}\right\}. \label{eq:cvar}
\end{align}
Key to our analysis is the \emph{dual representation} of CRMs. It is known that any CRM $\R\colon \mathcal{L}^{1}(\Omega)  \rightarrow \reals\cup\{+\infty\}$ can be expressed as the \emph{support function} of some closed convex set $\cQ \subseteq \mathcal{L}^1(\Omega)$ \citep{rockafellar2013fundamental}; that is, for any real random variable $Z\in \mathcal{L}^{1}(\Omega)$, we have
\begin{align}
\R[Z] &= \sup_{q\in \cQ} \E_P\left[ Z q \right] =   \sup_{q \in \cQ } \int_{\Omega}  Z(\omega)  q(\omega) \rmd \P(\omega). \quad \text{(dual representation)} \label{eq:supportint}
\end{align}
In this case, the set $\mathcal{Q}$ is called the \emph{risk envelope} associated with the risk measure $\R$. The risk envelope $\cQ_{\alpha}$ of $\cvar_{\alpha}[Z]$ is given by 
\begin{align}
\cQ_\alpha \coloneqq \left\{q\in  \mathcal{L}^1(\Omega) \ \left| \  \exists Q \in \mathcal{M}_P(\Omega),\  q= \frac{\rmd \Q}{\rmd \P} \leq \frac{1}{\alpha} \right. \right\}, \label{eq:theQset}
\end{align}
and so substituting $\cQ_{\alpha}$ for $\cQ$ in \eqref{eq:supportint} yields $\cvar_{\alpha}[Z]$. Though the overall approach we take in this paper may be generalizable to other popular CRMs, (see Appendix~\ref{app:gen}) we focus our attention on $\cvar$ for which we derive new PAC-Bayesian and concentration bounds in terms of its natural estimator $\what\cvar_{\alpha}[Z]$; given i.i.d.~copies of $Z_1,\dots,Z_n$ of $Z$, we define
\begin{align}
\what\cvar_{\alpha}[Z] \coloneq \what \C_{\alpha}[Z] \coloneqq \inf_{\mu \in \reals} \left\{   \mu +\sum_{i=1}^n \frac{[Z_i-\mu]_+}{n \alpha}  \right \}. \label{eq:empcvar}
\end{align}
From now on, we write $\C_{\alpha}[Z]$ and $\what\C_{\alpha}[Z]$ for $\cvar_{\alpha}[Z]$ and $\what\cvar_{\alpha}[Z]$, respectively.

\section{PAC-Bayesian Bound for the Conditional Value at Risk}\label{sec:def}
In this section, we briefly describe the statistical learning setting, formulate our goal, and present our main results. 

In the statistical learning setting, $Z$ is a \emph{loss} random variable which can be written as $Z=\ell(h,X)$, where $\ell\colon \mathcal{H}\times\mathcal{X}\rightarrow \reals_{\geq 0}$ is a loss function and $\mathcal{X}$ [resp. $\mathcal{H}$] is a feature [resp. hypotheses] space. The aim is to learn an hypothesis $\hat h = \hat h (X_{1:n}) \in \mathcal{H}$, or more generally a distribution $\rhohat =\rhohat (X_{1:n})$ over $\mathcal{H}$ (also referred to as randomized estimator), based on i.i.d.~samples $X_{1},\dots,X_n$ of $X$ which minimizes some measure of risk---typically, the expected risk $\E_\P[\ell(\rhohat, X )]$, where $\ell(\rhohat,X) \coloneqq \E_{h\sim \rhohat}[\ell(h,X)]$. 

Our work is motivated by the idea of replacing this expected risk by any coherent risk measure $\R$. In particular, if $\mathcal{Q}$ is the risk envelope associated with $\R$, then our quantity of interest is 
\begin{align}\R[\ell(\rhohat,X)]&\coloneqq  \sup_{q \in \cQ} \int_{\Omega}  \ell(\rhohat,X(\omega))  q(\omega) \rmd \P(\omega). \label{eq:risk}
\end{align}
Thus, given a consistent estimator $\what\R[\ell(\rhohat,X)]$ of $\R[\ell(\rhohat,X)]$ and some prior distribution $\uprho_0$ on $\cH$, our grand goal (which goes beyond the scope of this paper) is to bound the risk $\R[\ell(\rhohat,X)]$ as 
\begin{align}
\R[\ell(\rhohat,X)]  \leq  \what\R[\ell(\rhohat,X)] + \wtilde{\cO} \left( \sqrt{\frac{ \op{KL}(\rhohat\|\uprho_0)}{n}} \right),  \label{eq:pacbayes}
\end{align}
with high probability. Based on \eqref{eq:empcvar}, the consistent estimator we use for $\C_{\alpha}[\ell(\rhohat, X)]$ is
\begin{align}
 \what \C_{\alpha}[\ell(\rhohat, X)] \coloneqq \inf_{\mu \in \reals} \left\{   \mu +\sum_{i=1}^n \frac{[\ell(\rhohat, X_i)-\mu]_+}{n \alpha}  \right \}, \quad \alpha \in(0,1). \label{eq:empcvar2}
\end{align} 
This is in fact a consistent estimator (see e.g. \citep[Proposition 9]{duchi2018learning}). As a first step towards the goal in \eqref{eq:pacbayes}, we derive a sharp PAC-Bayesian bound for the conditional value at risk, which we state now as our main theorem:
\begin{theorem}
	\label{thm:main}
	Let $\alpha\in(0,1)$, $\delta\in(0,{1}/{2})$, $n\geq 2$, and $N \coloneqq \ceil{\log_2(n/\alpha)}$. Further, let $\uprho_0$ be any distribution on a hypothesis set $\cH$, $\ell\colon \cH \times \cX \rightarrow [0,1]$ be a loss, and $X_1,\dots, X_n$ be i.i.d.~copies of $X$. Then, for any ``posterior'' distribution $\rhohat=\rhohat(X_{1:n})$ over $\cH$, $\varepsilon_n\coloneqq \sqrt{\frac{\ln (N /\delta)}{2 \alpha n}} +\frac{\ln (N /{\delta})}{3 \alpha n}$, and $\comp\coloneqq \KL(\rhohat \| \uprho_0)+\ln (N/{\delta})$, we have, with probability at least $1-2\delta$.
	\begin{gather}
	\E_{h\sim \rhohat}	[\C_{\alpha}[\ell(h,X)]]\  \leq \  \what\C_{\alpha}[\ell(\rhohat,X)]+  \sqrt{\frac{27 \what\C_{\alpha}[\ell(\rhohat,X)]  \comp}{5 \alpha n}} + 2\varepsilon_n \what\C_{\alpha}[\ell(\rhohat,X)]  +  \frac{27\comp}{5n\alpha }. \label{eq:pacbayes0}
	\end{gather}
\end{theorem}
\paragraph{Discussion of the bound.}  Although we present the bound for the bounded loss case, our result easily generalizes to the case where $\ell(h,X)$ is sub-Gaussian or sub-exponential, for all $h\in \cH$. We discuss this in Section~\ref{sec:concentration}. Our second observation is that since $\C_{\alpha}[Z]$ is a coherent risk measure, it is convex in $Z$ \citep{rockafellar2013fundamental}, and so we can further bound the term $\E_{h\sim \rhohat}	[\C_{\alpha}[\ell(h,X)]]$ on the LHS of \eqref{eq:pacbayes0} from below by $\C_{\alpha}[\ell(\rhohat,X)] = \C_{\alpha}[\E_{h\sim \rhohat}[\ell(h,X)]]$. This shows that the type of guarantee we have in \eqref{eq:pacbayes0} is in general tighter than the one in \eqref{eq:pacbayes}.

Even though not explicitly done before,  a PAC-Bayesian bound of the form \eqref{eq:pacbayes} can be derived for a risk measure $\R$ using an existing technique due to \citet{mcallester2003} as soon as, for any fixed hypothesis $h$, the difference $\R[\ell(h,X)] - \what\R[\ell(h,X)]$ is sub-exponential with a sufficiently fast tail decay as a function of $n$ (see the proof of Theorem~1 in \citep{mcallester2003}). While it has been shown that the difference $\C_{\alpha}[Z]-\what\C_{\alpha}[Z]$ also satisfies this condition for bounded i.i.d.~random variables $Z,Z_1,\dots,Z_n$ (see \emph{e.g.} \cite{brown2007,wang2010}), applying the technique of \citet{mcallester2003} yields a bound on $\E_{h\sim \rhohat}[\C_{\alpha}[\ell(h,X)]]$ (\emph{i.e.} the LHS of \eqref{eq:pacbayes0}) of the form 
\begin{align}
\E_{h\sim \rhohat}[\what\C_{\alpha}[\ell(h,X)]] + \sqrt{\frac{\KL(\rhohat \| \uprho_0) + \ln \frac{n}{\delta} }{\alpha  n}}. \label{eq:alternative}
\end{align}
Such a bound is weaker than ours in two ways; (\textbf{I}) by Jensen's inequality the term $\what\C_{\alpha}[\ell(\rhohat,X)]$ in our bound (defined in \eqref{eq:empcvar2}) is always smaller than the term $\E_{h\sim \rhohat}[\what\C_{\alpha}[\ell(h,X)]]$ in \eqref{eq:alternative}; and (\textbf{II}) unlike in our bound in \eqref{eq:pacbayes0}, the complexity term inside the square-root in \eqref{eq:alternative} does not multiply the empirical conditional value at risk $\what\C_{\alpha}[\ell(\rhohat,X)]$. This means that our bound can be much smaller whenever $\what\C_{\alpha}[\ell(\rhohat,X)]$ is small---this is to be expected in the statistical learning setting since $\rhohat$ will typically be picked by an algorithm to minimize the empirical value $\what\C_{\alpha}[\ell(\rhohat,X)]$. This type of improved PAC-Bayesian bound, where the empirical error appears multiplying the complexity term inside the square-root, has been derived for the expected risk in works such as \citep{Seeger02,langford2003pac,catoni2007pac,maurer2004note}; these are arguably the state-of-the-art generalization bounds. 

\paragraph{A reduction to the expected risk.} A key step in the proof of Theorem~\ref{thm:main} is to show that for a real random variable $Z$ (not necessarily bounded) and $\alpha \in(0,1)$, one can construct a function $g \colon \reals \rightarrow \reals$ such that the auxiliary variable $Y=g(Z)$ satisfies \textbf{(I)} \begin{align}\E[Y]= \E[g(Z)]= \C_{\alpha}[Z];\label{eq:firststep} \end{align} and \textbf{(II)} for i.i.d.~copies $Z_{1:n}$ of $Z$, the i.i.d.~random variables $Y_{1}\coloneqq g(Z_1), \dots, Y_n\coloneqq g(Z_n)$ satisfy \begin{align} \frac{1}{n} \sum_{i=1}^n Y_i \leq \what\C_{\alpha}[Z](1+\epsilon_n),  \quad  \text{where} \ \ \epsilon_n = \wtilde{\cO}(\alpha^{-1/2} n^{-1/2}),  \label{eq:secondstep} \end{align} 
with high probability. Thus, due to \eqref{eq:firststep} and \eqref{eq:secondstep}, bounding the difference  
\begin{align}
\E[Y] - \frac{1}{n} \sum_{i=1}^n Y_i, 
\label{eq:reduction}
\end{align}
is sufficient to obtain a concentration bound for $\cvar$. Since $Y_1, \dots,Y_n$ are i.i.d., one can apply standard concentration inequalities, which are available whenever $Y$ is sub-Gaussian or sub-exponential, to bound the difference in \eqref{eq:reduction}. Further, we show that whenever $Z$ is sub-Gaussian or sub-exponential, then essentially so is $Y$. Thus, our method allows us to obtain concentration inequalities for $\what\C_{\alpha}[Z]$, even when $Z$ is unbounded. We discuss this in Section~\ref{sec:concentration}.

\section{Proof Sketch for Theorem~\ref{thm:main}}\label{sec:PACBayes}
In this section, we present the key steps taken to prove the bound in Theorem~\ref{thm:main}. We organize the proof in three subsections. In Subsection~\ref{sec:auxe}, we introduce an auxiliary estimator $\wtilde\C_{\alpha}[Z]$ for $\C_{\alpha}[Z]$, $\alpha \in(0,1)$, which will be useful in our analysis; in particular, we bound this estimator in terms of $\what\C_{\alpha}[Z]$ (as in \eqref{eq:secondstep} above, but with the LHS replaced by $\wtilde\C_{\alpha}[Z]$). In Subsection~\ref{sec:auxy}, we introduce an auxiliary random variable $Y$ whose expectation equals $\C_{\alpha}[Z]$ (as in \eqref{eq:firststep}) and whose empirical mean is bounded from above by the estimator $\wtilde\C_{\alpha}[Z]$ introduced in Subsection~\ref{sec:auxe}---this enables the reduction described at the end of Section~\ref{sec:def}. In Subsection~\ref{sec:pac}, we conclude the argument by applying the classical Donsker-Varadhan variational formula \citep{donsker1976asymptotic,C75}.
\subsection{An Auxiliary Estimator for $\cvar$} 
\label{sec:auxe}
In this subsection, we introduce an auxiliary estimator $\wtilde\C_{\alpha}[Z]$ of $\C_{\alpha}[Z]$ and show that it is not much larger than $\what\C_{\alpha}[Z]$. For $\alpha , \delta\in(0,1)$, $n\in\mathbb{N}$, and $\uppi \coloneqq (1,\dots,1)^\top/n\in \reals^n$, define:
\begin{align}
&\hspace{-0.2cm}\widetilde{\cQ}_{\alpha}  \coloneqq \left\{\bm{q}\in [0,{1}/{\alpha}]^n \  :\     |\E_{i\sim \uppi}[q_i]-1|  \leq \epsilon_{n}  \right\}, \quad    \text{where}  \quad \epsilon_{n} \coloneqq \sqrt{\frac{\ln \frac{1}{\delta}}{2 \alpha n}} +\frac{\ln \frac{1}{\delta}}{3 \alpha n}.  \label{eq:espilon}
\end{align}
Using the set $\wtilde \cQ_{\alpha}$, and given i.i.d.~copies $Z_1,\dots, Z_n$ of $Z$, let
\begin{align}
\wtilde\C_{\alpha}[Z]\coloneqq \sup_{\bm q \in \widetilde{\cQ}_{\alpha}}\frac{1}{n} \sum_{i=1}^n Z_i q_i. \label{eq:newestim}
\end{align}
In the next lemma, we give a ``variational formulation'' of $\wtilde\C_{\alpha}[Z]$, which will be key in our results:
\begin{lemma}
\label{lem:keylem}
Let $\alpha,\delta\in(0,1)$, $n\in \mathbb{N}$, and $ \wtilde{\C}_{\alpha}[Z]$ be as in \eqref{eq:newestim}. Then, for any $Z_1,\dots, Z_n\in \reals$,
\begin{align}
\hspace{-0.2cm}
\wtilde\C_{\alpha}[Z] = \inf_{\mu \in \reals}\left\{ \mu + |\mu| \epsilon_n +\frac{\E_{i\sim\uppi}[Z_i-\mu]_+}{\alpha} \right\}, \quad \text{where $\epsilon_{n}$ as in \eqref{eq:espilon}.} \label{eq:obj}
\end{align}
\end{lemma}
The proof of Lemma~\ref{lem:keylem} (which is in Appendix \ref{app:proofs}) is similar to that of the generalized Donsker-Varadhan variational formula considered in \citep{Beck2003}.
The ``variational formulation'' on the RHS of \eqref{eq:obj} reveals some similarity between $\wtilde\C_{\alpha}[Z]$ and the standard estimator $\what{\C}_\alpha[Z]$ defined in \eqref{eq:empcvar}. In fact, thanks to Lemma \ref{lem:keylem}, we have the following relationship between the two: 
\begin{lemma} \label{lem:relation}Let $\alpha,\delta\in(0,1)$, $n\in \mathbb{N}$, and $Z_1,\dots,Z_n \in\reals_{\geq 0} $. Further, let $Z_{(1)},\dots, Z_{(n)}$ be the decreasing order statistics of $Z_1,\dots,Z_n$. Then, for $\epsilon_n$ as in \eqref{eq:espilon}, we have
\begin{align}
\wtilde\C_{\alpha}[Z] \leq \what\C_{\alpha}[Z]\cdot (1 + \epsilon_n);  \label{eq:newestimator}
\intertext{and if $Z_1,\dots, Z_n \in \reals$ (not necessarily positive), then}
\wtilde\C_{\alpha}[Z] \leq \what\C_{\alpha}[Z] +   |Z_{(\ceil{n\alpha})}| \cdot  \epsilon_n. \label{eq:orderstat}
\end{align}
\end{lemma}
The inequality in \eqref{eq:orderstat} will only be relevant to us in the case where $Z$ maybe negative, which we deal with in Section \ref{sec:concentration} when we derive new concentration bounds for $\cvar$.
\subsection{An Auxiliary Random Variable} 
\label{sec:auxy}
In this subsection, we introduce a random variable $Y$ which satisfies the properties in \eqref{eq:firststep} and \eqref{eq:secondstep}, where $Y_1,\dots, Y_n$ are i.i.d.~copies of $Y$ (this is where we leverage the dual representation in \eqref{eq:supportint}). This allows us to the reduce the problem of estimating $\cvar$ to that of estimating an expectation.

Let $\cX$ be an arbitrary set, and $f\colon \cX \rightarrow \reals$ be some fixed measurable function (we will later set $f$ to a specific function depending on whether we want a new concentration inequality or a PAC-Bayesian bound for $\cvar$). Given a random variable $X$ in $\cX$ (arbitrary for now), we define \begin{gather} Z  \coloneqq f(X) \label{eq:theZ} 
\intertext{and the auxiliary random variable:}
Y \coloneqq Z \cdot \E[q_\star \mid X] = f(X) \cdot \E[q_\star \mid X], \quad  \text{where} \quad 
q_{\star}\in \argmax_{q\in \cQ_\alpha} \E[Z q], \label{eq:newvar1}
\end{gather} 
and $\cQ_\alpha$ as in \eqref{eq:theQset}. In the next lemma, we show two crucial properties of the random variable $Y$---these will enable the reduction mentioned at the end of Section~\ref{sec:def}:
\begin{lemma}
\label{lem:newvar} Let $\alpha,\delta\in(0,1)$ and $X_1,\dots, \X_n$ be i.i.d.~random variables in $\cX$. Then, \textbf{(I)} the random variable $Y$ in \eqref{eq:newvar1} and $Y_i\coloneqq Z_i \cdot \E[q_\star\mid \X_i], i\in[n]$, $\text{where} \ Z_i \coloneqq f(X_i),$ are i.i.d.~and satisfy $\E[Y]= \E[Y_i] = \C_{\alpha}[Z]$, for all $i\in[n]$; and \textbf{(II)} with probability at least $1-\delta$, \begin{align} (\E[q_\star \mid \X_1],\dots,\E[q_\star \mid \X_n])^\top\in \widetilde{Q}_{\alpha},\label{eq:decond} \quad \text{where $\widetilde{\cQ}_{\alpha}$ is as in \eqref{eq:espilon}.} \end{align}
\end{lemma}
The random variable $Y$ introduced in \eqref{eq:newvar1} will now be useful since due to \eqref{eq:decond} in Lemma~\ref{lem:newvar}, we have, for $\alpha,\delta\in(0,1)$; $Z$ as in \eqref{eq:theZ}; and i.i.d.~random variables $\X, \X_1,\dots,\X_n\in \cX$, 
\begin{align}
P \left[\frac{1}{n} \sum_{i=1}^n Y_i \leq  \wtilde\C_{\alpha}[Z]\right] \geq 1-\delta, \quad  \text{where $Y_i= Z_i \cdot \E[q_\star\mid \X_i]$} \label{eq:ineqestimator}
\end{align}
and $\wtilde\C_{\alpha}[Z]$ as in \eqref{eq:newestim}. We now present a concentration inequality for the random variable $Y$ in \eqref{eq:newvar1}; the proof, which can be found in Appendix \ref{sec:appendix}, is based on a version of the standard Bernstein's moment inequality \citep[Lemma A.5]{cesa2006prediction}: 
\begin{lemma}
\label{lem:conc}
Let $\X,(\X_{i})_{i\in[n]}$ be i.i.d.~random variables in $\cX$. Further, let $Y$ be as in \eqref{eq:newvar1}, and $Y_i = f(X_i) \cdot \E[q_\star \mid \X_i], i\in [n]$, with $q_\star$ as in \eqref{eq:newvar1}. If $\{f(x)\mid x\in \cX\} \subseteq [0,1]$, then for all $\eta \in[0,\alpha]$,
\begin{align}
\E \left [\exp\left( n \eta \cdot \left( \E_P[Y] - \frac{1}{n} \sum_{i=1}^n Y_i  -   \frac{\eta \kappa({\eta}/{\alpha})}{\alpha} \C_{\alpha}[Z]  \right)\right) \right] \leq 1, \quad \text{where $Z = f(X)$},
\label{eq:intermediate0}
\end{align}
and $\kappa(x)\coloneqq (e^x - 1-x)/x^2$, for $x\in \reals$.
\end{lemma}
Lemma~\ref{lem:conc} will be our starting point for deriving the PAC-Bayesian bound in Theorem~\ref{thm:main}.
\subsection{Exploiting the Donsker-Varadhan  Formula}
\label{sec:pac}
In this subsection, we instantiate the results of the previous subsections with $f(\cdot)\coloneqq \ell(\cdot,h)$, $h\in \cH$, for some loss function $\ell \colon \cH \times \cX \rightarrow [0,1]$; in this case, the results of Lemmas~\ref{lem:relation} and \ref{lem:conc} hold for
\begin{align}Z=Z_h\coloneqq  \ell(h,X), \label{eq:newZ} \end{align}
for any hypothesis $h\in\cH$. Next, we will need the following result which follows from the classical Donsker-Varadhan variational formula \citep{donsker1976asymptotic,C75}:
\begin{lemma} \label{prop:donsker}
Let $\delta \in(0,1)$, $\gamma > 0$ and $\uprho_0$ be any fixed (prior) distribution over $\cH$. Further, let $\{R_h: h \in \cH\}$ be any family of random variables such that $\E [\exp( \gamma R_h)] \leq 1$, for all $h \in \cH$. Then, for any (posterior) distribution $\rhohat$ over $\cH$, we have  
\begin{align}P\left[ \E_{h \sim \rhohat}[R_h] \leq  \frac{\KL(\rhohat \| \uprho_0 ) +\ln \frac{1}{\delta}}{\gamma}\right]\geq 1-\delta. \label{eq:PACBayes0}  \end{align} 
\end{lemma}
In addition to $Z_h$ in \eqref{eq:newZ}, define $Y_h \coloneqq \ell(h,X) \cdot \E[q_\star \mid X]$ and $Y_{h,i} \coloneqq \ell(h,X_i) \cdot \E[q_\star \mid X_i]$, for $i\in[n]$. Then, if we set $\gamma=\eta n$ and $R_h=  \E_P[Y_h] - \sum_{i=1}^n Y_{h,i}/n  -  \eta \kappa({\eta}/{\alpha})\C_{\alpha}[Z_h]/{\alpha}$, Lemma \ref{lem:conc} guarantees that $\E[\exp(\gamma R_h)]\leq 1$, and so by Lemma \ref{prop:donsker} we get the following result:
\begin{theorem}
\label{thm:prepac}
Let $\alpha,\delta\in(0,1)$, and $\eta \in[0,\alpha]$. Further, let $X_1,\dots, X_n$ be i.i.d.~random variables in $\cX$. Then, for any randomized estimator $\rhohat =\rhohat(X_{1:n})$ over $\cH$, we have, with $\what Z\coloneqq  \E_{h\sim \rhohat}[\ell(h,X)]$,
\begin{align}
&\hspace{-0.5cm}\E_{h\sim \rhohat}[\C_{\alpha}[\ell(h,X)]]  \leq \what\C_{\alpha}[\what Z] (1+\epsilon_n)+  \frac{\eta \kappa({\eta}/{\alpha})\E_{h\sim \rhohat}[\C_{\alpha}[\ell(h,X)]] }{\alpha } + \frac{\KL(\rhohat \| \uprho_0)+ \ln \frac{1}{\delta}}{\eta n} , \label{eq:goal1}
\end{align}
with probability at least $1-2\delta$ on the samples $X_{1},\dots,X_n$, where $\epsilon_n$ is defined in \eqref{eq:espilon}.
\end{theorem}
If we could optimize the RHS of \eqref{eq:goal1} over $\eta \in[0,\alpha]$, this would lead to our desired bound in Theorem~\ref{thm:main} (after some rearranging). However, this is not directly possible since the optimal $\eta$ depends on the sample $X_{1:n}$ through the term $\KL(\rhohat\| \uprho_0)$. The solution is to apply the result of Theorem~\ref{thm:prepac} with a union bound, so that \eqref{eq:goal1} holds for any estimator $\hat\eta = \hat \eta(X_{1:n})$ taking values in a carefully chosen grid $\cG$; to derive our bound, we will use the grid $
\cG \coloneqq \left\{ {\alpha}{ 2^{-1}},\dots,  {\alpha}{ 2^{-N}} \mid N\coloneqq \ceil{{1}/{2}  \log_2 ({n}/{\alpha})} \right\}.$
From this point, the proof of Theorem~\ref{thm:main} is merely a mechanical exercise of rearranging \eqref{eq:goal1} and optimizing $\hat \eta$ over $\cG$, and so we postpone the details to Appendix \ref{sec:appendix}.

\section{New Concentration Bounds for $\cvar$}\label{sec:concentration}
In this section, we show how some of the results of the previous section can be used to reduce the problem of estimating $\C_{\alpha}[Z]$ to that of estimating a standard expectation. This will then enable us to easily obtain concentration inequalities for $\what\C_{\alpha}[Z]$ even when $Z$ is unbounded. We note that previous works \citep{kolla2019concentration, bhat2019concentration} used sophisticated techniques to deal with the unbounded case (sometimes achieving only sub-optimal rates), whereas we simply invoke existing concentration inequalities for empirical means thanks to our reduction. 

The key results we will use are Lemmas \ref{lem:relation} and \ref{lem:newvar}, where we instantiate the latter with $\cX=\reals$ and $f\equiv \op{id}$, in which case: \begin{align}Y=Z\cdot \E[q_\star\mid Z], \ \ \text{and } q_\star\in \argmax_{q\in \cQ_{\alpha}} \E[Z q].\label{eq:newvar2} \end{align}  Together, these two lemmas imply that, for any $\alpha, \delta\in(0,1)$,  i.i.d.~random variables $Z_1,\dots,Z_n$, 
\begin{align}
\hspace{-0.5cm}\C_{\alpha}[Z] - \what\C_{\alpha}[Z]-\left|Z_{(\ceil{n\alpha})}\right| \epsilon_n \leq \E[Y] - \frac{1}{n} \sum_{i=1}^n Y_i, \label{eq:reduc}
\end{align}
with probability at least $1-\delta$, where $\epsilon_n$ is as in \eqref{eq:espilon} and $Z_{(1)},\dots, Z_{(n)}$ are the decreasing order statistics of $Z_1,\dots,Z_n\in \reals$. Thus, getting a concentration inequality for $ \what\C_{\alpha}[Z]$ can be reduced to getting one for the empirical mean $\sum_{i=1}^n Y_i/n$ of the i.i.d.~random variables $Y_1,\dots,Y_n$. Next, we show that whenever $Z$ is a sub-exponential [resp. sub-Gaussian] random variable, essentially so is $Y$. But first we define what this means:
\begin{definition}
Let $\mathcal{I}\subseteq \reals$, $b>0$, and $Z$ be a random variable such that, for some $\sigma >0$,
\begin{align}
\E [ \eta \cdot (  Z-\E[Z])] \leq \exp\left(\eta^2 \sigma^2/2 \right), \quad \forall \eta \in \mathcal{I},
\end{align} 
Then, $Z$ is $(\sigma,b)$-sub-exponential [resp. $\sigma$-sub-Gaussian] if $\mathcal{I} =(-{1}/{b},{1}/{b})$ [resp. $\mathcal{I}=\reals$].
\end{definition}
\begin{lemma}
\label{lem:lemmas}
Let $\sigma, b>0$ and $\alpha \in(0,1)$. Let $Z$ be a zero-mean real random variable and let $Y$ be as in \eqref{eq:newvar2}. If $Z$ is $(\sigma,b)$-sub-exponential [resp. $\sigma$-sub-Gaussian], then  
\begin{align}
\E[\exp( \eta Y  ) ]\leq 2 \exp({\eta^2 \sigma^2}/({2 \alpha^2})), \quad \forall \eta \in  (-{\alpha}/{b},{\alpha}/{b}) \quad [\text{resp. }  \eta\in\reals]. \label{eq:subexp}
\end{align}
\end{lemma}
Note that in Lemma~\ref{lem:lemmas} we have assumed that $Z$ is a zero-mean random variable, and so we still need to do some work to derive a concentration inequality for $\what\C_{\alpha}[Z]$. In particular, we will use the fact that $\C_{\alpha}[Z-\E[Z]] =\C_{\alpha}[Z]- \E[Z]$ and $\what\C_{\alpha}[Z-\E[Z]] =\what\C_{\alpha}[Z]- \E[Z]$, which holds since $\C_{\alpha}$ and $\what\C_{\alpha}$ are coherent risk measures, and thus translation invariant (see Definition~\ref{def:coherent}). We use this in the proof of the next theorem (which is in Appendix \ref{sec:appendix}):
\begin{theorem}
\label{thm:boundsunbounded}
Let $\sigma, b>0$, $\alpha,\delta  \in(0,1)$, and $\epsilon_n$ be as in \eqref{eq:espilon}. If $Z$ is a $\sigma$-sub-Gaussian random variable, then with ${\textsc{G}}[Z] \coloneqq  \C_{\alpha}[Z] - \what\C_{\alpha}[Z]$ and $t_n \coloneqq \left|Z_{(\ceil{n\alpha })} -\E[Z]\right|\cdot \epsilon_n$, we have 
\begin{gather}
P\left[ \textsc{G}[Z] \geq t+t_n \right]\leq \delta + 2 \exp(-n \alpha^2 t^2/(2 \sigma^2)), \ \  \forall t\geq 0; \label{eq:gaussian}
\shortintertext{otherwise, if $Z$ is $(\sigma,b)$-sub-exponential random variable, then}
\hspace{-0.2cm}P\left[ \textsc{G}[Z] \geq t+t_n \right]\leq \delta + \left\{ \begin{array}{ll} 2\exp\left(-n \alpha^2 t^2/(2 \sigma^2)\right),  & \text{if}\  0 \leq t \leq  { \sigma^2}/({b \alpha});  \\ 2\exp\left(-n \alpha t/(2 b)\right), & \text{if}\ t >  { \sigma^2}/({b \alpha}).
\end{array} \right. \nonumber
\end{gather}
\end{theorem}
We note that unlike the recent results due to \citet{bhat2019concentration} which also deal with the unbounded case, the constants in our concentration inequalities in Theorem~\ref{thm:boundsunbounded} are explicit. 

When $Z$ is a $\sigma$-sub-Gaussian random variable with $\sigma>0$, an immediate consequence of Theorem \ref{thm:boundsunbounded} is that by setting $t= \sqrt{{2 \sigma^2 \ln({2}/{\delta})}/({n\alpha^2})}$ in \eqref{eq:gaussian}, we get that, with probability at least $1-2\delta$,
\begin{gather} 
\C_{\alpha}[Z] -\what\C_{\alpha}[Z]  \leq \frac{\sigma}{\alpha}\sqrt{\frac{2\ln \frac{1}{\delta}}{n}}+\left|Z_{(\ceil{n\alpha })} -\E[Z]\right|\cdot  \left(\sqrt{\frac{ \ln \frac{1}{\delta} }{2\alpha n}} + \frac{ \ln \frac{1}{\delta} }{3\alpha n} \right). \label{eq:dependence}
\end{gather}
A similar inequality holds for the sub-exponential case. We note that the term $|Z_{(\ceil{n\alpha})} -\E[Z]|$ in \eqref{eq:dependence} can be further bounded from above by \begin{align}\frac{{n \alpha}}{\floor{n\alpha}}  \what \C_{\alpha}[Z]-\E_{\what P_n}[Z] + \left|\E[Z]-\E_{\what P_n}[Z]\right|. \label{eq:upperbound}\end{align}
This follows from the triangular inequality and facts that $\what\C_{\alpha}[Z]\geq \frac{1}{n\alpha}\sum_{i=1}^{\floor{n \alpha}} Z_{(i)}\geq \frac{\floor{n \alpha}}{n \alpha}  Z_{(\ceil{n\alpha})}$ (see \emph{e.g.} Lemma 4.1 in \cite{brown2007}), and $\what\C_{\alpha}[Z]\geq \E_{\what P_n}[Z]$ \citep{ahmadi2012entropic}. The remaining term $|\E_P[Z]-\E_{\what P_n}[Z]|$ in \eqref{eq:upperbound} which depends on the unknown $P$ can be bounded from above using another concentration inequality.

Generalization bounds of the form \eqref{eq:pacbayes} for unbounded but sub-Gaussian or sub-exponential $\ell(h,X)$, $h\in \cH$, can be obtained using the PAC-Bayesian analysis of \citep[Theorem 1]{mcallester2003} and our concentration inequalities in Theorem~\ref{thm:boundsunbounded}. However, due to the fact that $\alpha$ is squared in the argument of the exponentials in these inequalities (which is also the case in the bounds of \citet{bhat2019concentration,kolla2019concentration}) the generalization bounds obtained this way will have the $\alpha$ outside the square-root ``complexity term''---unlike our bound in Theorem~\ref{thm:main}.

We conjecture that the dependence on $\alpha$ in the concentration bounds of Theorem~\ref{thm:boundsunbounded} can be improved by swapping $\alpha^2$ for $\alpha$ in the argument of the exponentials; in the sub-Gaussian case, this would move $\alpha$ inside the square-root on the RHS of \eqref{eq:dependence}. We know that this is at least possible for bounded random variables as shown in \cite{brown2007,wang2010}. We now recover this fact by presenting a new concentration inequality for $\what\C_{\alpha}[Z]$ when $Z$ is bounded using the reduction described at the beginning of this section.
\begin{theorem}
\label{thm:cvarconc}
Let $\alpha,\delta \in(0,1)$, and $Z_{1:n}$ be i.i.d.~\emph{rv}s in $[0,1]$. Then, with probability at least $1-2 \delta$,
\begin{align}
\hspace{-0.5cm}\C_{\alpha}[Z] - \what \C_{\alpha}[Z]  \leq   \sqrt{ \frac{12  \C_{\alpha}[Z] \ln \frac{1}{\delta}}{5\alpha n}}\vee \frac{3\ln \frac{1}{\delta}}{\alpha n}   + \C_{\alpha}[Z]\left(\sqrt{\frac{\ln \frac{1}{\delta}}{2 \alpha n}} +\frac{ \ln \frac{1}{\delta}}{3 \alpha n}\right).\label{eq:cvarconc}
\end{align}
\end{theorem} 
The proof is in Appendix \ref{sec:appendix}. The inequality in \eqref{eq:cvarconc} essentially replaces the range of the random variable $Z$ typically present under the square-root in other concentration bounds \citep{brown2007,wang2010} by the smaller quantity $\C_{\alpha}[Z]$. The concentration bound \eqref{eq:cvarconc} is not immediately useful for computational purposes since its RHS depends on $\C_{\alpha}[Z]$. However, it is possible to rearrange this bound so that only the empirical quantity $\what\C_{\alpha}[Z]$ appears on the RHS of \eqref{eq:cvarconc} instead of $\C_{\alpha}[Z]$; we provide the means to do this in Lemma~\ref{lem:conv} in the appendix.

\section{Conclusion and Future Work}\label{sec:entropies}
In this paper, we derived a first PAC-Bayesian bound for $\cvar$ by reducing the task of estimating $\cvar$ to that of merely estimating an expectation (see Section \ref{sec:PACBayes}). This reduction then made it easy to obtain concentration inequalities for $\cvar$ (with explicit constants) even when the random variable in question is unbounded (see Section \ref{sec:concentration}). 

We note that the only steps in the proof of our main bound in Theorem \ref{thm:main} that are specific to $\cvar$ are Lemmas \ref{lem:keylem} and \ref{lem:relation}, and so the question is whether our overall approach can be extended to other coherent risk measures to achieve \eqref{eq:pacbayes}.

In Appendix \ref{app:gen}, we discuss how our results may be extended to a rich class of coherent risk measures known as $\varphi$-entropic risk measures. These CRMs are often used in the context of robust optimization \cite{namkoong2017}, and are perfect candidates to consider next in the context of this paper.



\DeclareRobustCommand{\VAN}[3]{#3} 
\bibliography{biblio}
\bibliographystyle{plainnat}

\clearpage
\appendix

\section{Proofs}\label{sec:appendix}
	\subsection{Proof of Lemma~\ref{lem:keylem}}
\label{app:proofs}

\begin{proof}Let $\eff(\cdot)\coloneqq \iota_{[0,1/\alpha]}(\cdot)$, where for a set $\mathcal{C} \subseteq \reals$, $\iota_{\mathcal{C}}(x)=0$ if $x\in \mathcal{C}$; and $+\infty$ otherwise. From \eqref{eq:newestim}, we have that $\wtilde{\C}_{\alpha}[Z]$ is equal to 
	\begin{align}
		\mathscr{P} \coloneqq \sup_{\bm{q}:|\E_{i\sim \uppi} [q_i]-1|\leq \epsilon_n } \E_{i\sim \uppi}[Z_i q_i  - \eff(q_i)], \label{eq:primal}
	\end{align}
where we recall that $\uppi = (1, \dots, 1)^\top/n \in \reals^n$. The Lagrangian dual $\mathscr{D}$ of \eqref{eq:primal} is given by 
	\begin{align}
		\mathscr{D} &\coloneqq \inf_{\eta,\gamma\geq 0} \left\{ \eta - \gamma  +   (\eta +\gamma) \epsilon_n + \sup_{\bm{q}\colon 0 \leq  q_i\leq 1/\alpha, i\in[n]} \left\{  \E_{i\sim\uppi}[(Z_i-\eta+\gamma) q_i-\eff(q_i)]    \right\}     \right\},\\
		& = \inf_{\eta,\gamma\geq 0} \left\{ \eta - \gamma  +   (\eta +\gamma) \epsilon_n +\E_{i\sim\uppi}\left[  \sup_{0 \leq  x\leq 1/\alpha }  \left\{ (Z_i-\eta+\gamma)x-\eff(x) \right\}  \right] \right\}, \label{eq:mid} \\
		& = \inf_{\eta,\gamma\geq 0} \left\{ \eta - \gamma  +   (\eta +\gamma) \epsilon_n +\E_{i\sim \uppi}[ \eff^\star(Z_i-\eta +\gamma)]\right\},\label{eq:inf} \\
		& = \inf_{\mu\in \reals} \left\{\mu + |\mu| \epsilon_n  + \E_{i\sim\uppi}[\varphi^\star(Z_i-\mu)] \right\}, \label{eq:laste}
	\end{align}
	where \eqref{eq:inf} is due to $\{x \in \reals\mid \varphi(x) <+\infty \} = [0,1/\alpha]$, and \eqref{eq:laste} follows by setting $\mu\coloneqq \eta - \gamma$ and noting that the $\inf$ in \eqref{eq:inf} is always attained at a point $(\eta, \gamma)\in \reals_{\geq 0}^2$ satisfying $\eta\cdot \gamma =0$, in which case $\eta +\gamma =|\mu|$; this is true because by the positivity of $\epsilon_n$, if $\eta,\gamma >0$, then $(\eta +\gamma)\epsilon_n$ can always be made smaller while keeping the difference $\eta -\gamma$ fixed. Finally, since the primal problem is feasible---$\bm{q}=\uppi$ is a feasible solution---there is no duality gap (see the proof of \citep[Theorem 4.2]{Beck2003}), and thus the RHS of \eqref{eq:laste} is equal to $\mathscr{P}$ in \eqref{eq:primal}. The proof is concluded by noting that the Fenchel dual of $\varphi$ satisfies $\varphi^\star(z) = 0\vee (z/\alpha)$, for all $z\in \reals$.
\end{proof}

\subsection{Proof of Lemma \ref{lem:relation}}
\begin{proof}
Let $\what\mu$ be the $\argmin$ in $\mu\in \reals$ of the RHS of \eqref{eq:empcvar}.
By Lemma~\ref{lem:keylem}, we have
\begin{align}
\wtilde{\C}_{\alpha}[Z] &= \inf_{\mu \in \reals}\left\{ \mu + |\mu| \epsilon_n +\frac{\E_{i\sim \uppi}[Z_i-\mu]_+}{\alpha} \right\}, \\
& \leq \what\mu + |\what\mu| \epsilon_n +\frac{\E_{i\sim \uppi}[Z_i-\what\mu)]_+}{\alpha}, \\ 
& = \what\C_\alpha[Z] + |\what\mu| \epsilon_n.  \quad  (\text{by definition of $\what\mu$}) \label{eq:lastkey}
\end{align}
The inequality in \eqref{eq:orderstat} follows from \eqref{eq:lastkey} and the fact that $\what\mu = Z_{(\ceil{n \alpha})}$ (see proof of \citep[Proposition 4.1]{brown2007}).

Now we show \eqref{eq:newestimator} under the assumption that $Z_i\geq 0$, for all $i\in[n]$. Note that by definition $\what\C_{\alpha}[Z] = \what \mu + \frac{1}{\alpha}\E_{i\sim \uppi}[Z_i-\what\mu]_+ $, and so $\what\mu \leq \what\C_{\alpha}[Z]$. Furthermore, since $\alpha\in(0,1)$ and $Z_i\geq 0$, for $i\in[n]$, the RHS of \eqref{eq:empcvar} is a decreasing function of $\mu$ on $]-\infty,0]$, and thus $\what\mu \geq 0$ (since $\what\mu$ is the minimizer of \eqref{eq:empcvar}). Combining the fact that $0\leq \what\mu\leq \what\C_{\alpha}[Z]$ with \eqref{eq:lastkey} completes the proof. 
\end{proof}

\subsection{Proof of Lemma \ref{lem:newvar}}
\begin{proof}
The first claim follows by the fact that $\X_i,i\in[n]$, are i.i.d., and an application of the total expectation theorem. Now for the second claim, let $\Delta \coloneqq |\E_{\what P_n}[q_\star \mid \X] -1|$. Since $q_\star$ is a density, the total expectation theorem implies
\begin{align}
\Delta = |\E_{\what P_n}[q_\star\mid \X] - \E[\E[q_\star\mid \X]]|,
\end{align}
and so by Bennett's inequality (see \emph{e.g.} Theorem~3 in \cite{maurer2009empirical}) applied to the random variable $\E[q_\star\mid \X]$, we get that, with probability at least $1-\delta$,
\begin{align}
\Delta& \leq \sqrt{\frac{\Var[\E[q_\star\mid \X]] \ln \frac{1}{\delta} }{2 n}} + \frac{\|\E[q_\star \mid \X] \|_{\infty}  \ln \frac{1}{\delta}   }{3n}, \\
&   \leq \sqrt{\frac{\E[\E[q_\star\mid \X]^2] \ln \frac{1}{\delta} }{2 n}} + \frac{\|\E[q_\star \mid \X] \|_{\infty}  \ln \frac{1}{\delta}   }{3n}, \\
&   \leq \sqrt{\frac{\|\E[q_\star \mid \X] \|_{\infty} \ln \frac{1}{\delta} }{2 n}} + \frac{\|\E[q_\star \mid \X] \|_{\infty}  \ln \frac{1}{\delta}   }{3n},
\end{align} 
where the last inequality follows by the fact that $\E[\E[q_\star \mid \X]^2] \leq \E[\E[q_\star \mid \X]] \cdot \| \E[q_\star \mid \X] \|_{\infty}=\| \E[q_\star \mid \X]\|_{\infty} $, which holds since $\E[q_\star\mid \X]\geq 0$ and $\E[\E[q_\star\mid \X]] =\E[q_\star]=1$. The proof is concluded by the facts that $\|\E[q_\star\mid \X]\|_{\infty}\leq \| q_\star\|_{\infty}$ (by Jensen's inequality); $\| q\|_{\infty}\leq 1/\alpha$, for all $q\in \cQ_{\alpha}$ by definition; and $q_\star \in \cQ_{\alpha}$.
\end{proof}

\subsection{Proof of Lemma \ref{lem:conc}}
We need the following lemma in the proof of Lemma \ref{lem:conc}:
\begin{lemma}
\label{lem:bernstein0}
Let $S, S_1, \dots, S_n$ be i.i.d.~random variable such that $S\in[0,B]$, $ B>0$. We have,
\begin{align}
\E_P\left[ \exp  \left( n \eta  \E_\P[S] -\eta \sum_{i=1}^n S_i  - n\eta^2  \kappa(\eta B) \cdot \E_P[S^2]\right) \right] \leq 1,  \label{eq:bernstein0}
\end{align}
for all $\eta \in[0,1/B]$, where $\kappa (\eta) \coloneqq  ({e^\eta-\eta -1})/{\eta^2}$.
\end{lemma}
\begin{proof}
	The desired bound follows by the version of Bernstein's moment inequality in \citep[Lemma A.5]{cesa2006prediction} and \cite[Proposition 10-(b)]{mhammedi2019pac}.
	\end{proof}
\begin{proof}{\bf of Lemma \ref{lem:conc}}
By Lemma~\ref{lem:newvar}, the random variables $Y,Y_1,\dots,Y_n$ are i.i.d., and so the result of Lemma~\ref{lem:bernstein0} applies; this means that \eqref{eq:bernstein0} holds for $(S,S_1,\dots,S_n)=(Y,Y_1,\dots,Y_n)$ and $B = b \geq \|Y\|_{\infty}$. Thus, to complete the proof it suffices to bound $\|Y\|_{\infty}$ and $ \|Y\|_2^2 =\E[Y^2]$ from above. Starting with $\E[Y^2]$, and recalling that $Z=f(X)\in[0,1]$ by assumption, we have:
\begin{align}
\E[Y^2] &= \E[Z^2 \cdot \E[q_\star \mid X]^2],  \\  &  \leq \E[Z\cdot \E[q_\star \mid X]] \cdot  \|Z  \cdot  \E[q_\star \mid X] \|_{\infty},     (\text{H\"older})  \\
& \leq  \C_{\alpha}[Z]\cdot \|Z \cdot \E[q_\star \mid X] \|_{\infty} , \quad   \quad  (\text{Lemma~\ref{lem:newvar}})\\
& \leq \C_{\alpha}[Z]/\alpha,  \quad \quad   \quad  \quad  (Z \leq 1,\ q_\star \leq 1/\alpha) 
\end{align}
where the fact that $q_{\star}\leq 1/\alpha$ follows simply from $q_\star\in \cQ_{\alpha}$ and the definition of $\cQ_{\alpha}$. We also have
\begin{align}
\hspace{-0.2cm}\|Y\|_{\infty} = \|Z \cdot \E[q_\star\mid X]  \|_{\infty}& \leq \|Z\|_{\infty} \cdot  \|\E[q_\star\mid X]\|_{\infty},   \\ 
& \leq \|q_\star\|_{\infty},    (Z\leq 1 \ \text{\&} \ \text{Jensen})\\
& \leq 1/\alpha, 
\end{align}
again the last inequality follows from $q_\star\in \cQ_{\alpha}$ and the definition of $\cQ_{\alpha}$. 
\end{proof}

\subsection{Proof of Theorem \ref{thm:prepac}}
\begin{proof}
Let $h\in \cH$ and $\alpha,\delta\in(0,1)$, and define \begin{align}
R_h\coloneqq \C_{\alpha}[Z_h] - \frac{1}{n} \sum_{i=1}^n Y_i  -\frac{\eta\kappa({\eta/\alpha}) }{\alpha}\C_{\alpha}[Z_h],  \label{eq:Rhdef} 
\end{align}
where $Y_i \coloneqq \ell(h, X_i) \cdot \E[q_\star \mid X_i],i\in[n]$, where $q_\star$ is as in \eqref{eq:newvar1} with $Z$ as in \eqref{eq:newZ}. By Lemma~\ref{lem:newvar}, $\C_{\alpha}[Z_h]=\E_P[Y]$, where $Y\coloneqq \ell(h,X) \cdot \E[q_\star \mid X]$. Thus, by Lemma \ref{lem:conc} with $Z = Z_h$, we have $\E_P [\exp( n \eta R_h)]\leq 1$. 
Applying Lemma~\ref{prop:donsker} with $R_h$ as in \eqref{eq:Rhdef} and $\gamma =n \eta$, yields, 
\begin{align}
\E_{h\sim \rhohat}[\C_{\alpha}[\ell(h,X)]] & \leq\ \frac{1}{n} \sum_{i=1}^n \ell(\rhohat, X_i) \cdot \E[q_\star \mid X_i] +\frac{\eta \kappa({\eta/\alpha})\E_{h\sim \rhohat}[\C_{\alpha}[\ell(h,X)]] }{\alpha } \\ & \qquad+ \frac{\KL(\rhohat \| \uprho_0)+ \ln \frac{1}{\delta}}{\eta n} , \label{eq:midway}
\end{align}
with probability at least $1-\delta$. Now invoking Lemmas \ref{lem:relation} and \ref{lem:newvar} (in particular \eqref{eq:decond}), yields\begin{align}
\frac{1}{n}   \sum_{i=1}^n {\ell(\rhohat, X_i) \cdot \E[q_\star \mid X_i]}\geq \what\C_{\alpha}[\what Z]\cdot (1+\epsilon_n). \label{eq:secondbound}
\end{align}
with probability at least $1-\delta$, where $\what Z\coloneqq  \E_{h\sim \rhohat}[\ell(h,X)]$. Combining this with \eqref{eq:midway} via a union bound yields the desired bound. 
\end{proof}

						\subsection{Proof of Theorem~\ref{thm:main}}
		\label{sec:thmmain}
		To prove Theorem~\ref{thm:main}, we will need the following lemma:
		\begin{lemma}
			\label{lem:conv}
			Let $R,\what R, A,B>0$. If $R \leq \what R + \sqrt{R A} + B$, then
			\begin{align}
	R & \leq \what R +  \sqrt{\what R A} + 2 B + A.
			\end{align} 
		\end{lemma}
		\begin{proof}
			If $R \leq \what R + \sqrt{R A} + B$, then for all $\eta> 0$,
			\begin{gather}
			R \leq \what R + \frac{\eta}{2} R  +  \frac{A}{2\eta } + B,
			\shortintertext{which after rearranging, becomes,}
			R \leq \frac{\what R}{1-\eta/2} + \frac{A}{2\eta \cdot (1-\eta/2)} + \frac{B}{1-\eta/2}, \quad \text{	for $\eta\notin \{0,2\}$. } \label{eq:arranged}
			\end{gather}
		The minimizer of the RHS of \eqref{eq:arranged} is given by $$\eta =\eta_\star \coloneqq \frac{-A + \sqrt{A^2 + 4 AB + 4 A \what R}}{2(B+\what R)}.$$
			Plugging this $\eta$ into \eqref{eq:arranged}, yields, 
			\begin{align}
			R & \leq  \what R + \frac{A}{2} + B + \frac{1}{2} \sqrt{4 A \what R + A^2 + 4 A B}, \\
			& \leq \what R  + A + 2B + \sqrt{A \what R},\label{eq:lastone}
			\end{align}
			where \eqref{eq:lastone} follows by the facts that $A^2 + 4 A B \leq (A+2B)^2$ and $\sqrt{4\what R A+(A+2B)^2} \leq \sqrt{4 \what R A}+ A+2B$.
			\end{proof}
		\begin{proof}{\bf of Theorem~\ref{thm:main}}
			Define the grid $\cG$ by 
			\begin{align}
			\cG \coloneqq \left\{ 2^{-1} \alpha,\dots, 2^{-N} \alpha \mid N\coloneqq \ceil{1/2 \cdot \log_2 \tfrac{n}{\alpha}} \right\},
			\end{align}
			and let $\hat\eta = \hat \eta(Z_{1:n}) \in \cG$ be any estimator. Then, using the fact that $\kappa(x)\leq 3/5$, for all $x \leq 1/2$, and invoking Theorem~\ref{thm:prepac} with a union bound over $\eta \in \cG$, and $\varepsilon_n \coloneqq \sqrt{\frac{\ln \frac{N}{\delta}}{2 \alpha n}} +\frac{\ln \frac{N}{\delta}}{3 \alpha n}$, we get that 
			\begin{gather}
			\E_{h\sim \rhohat}	[\C_{\alpha}[\ell(h,X)]] - \what\C_{\alpha}[\what Z] \cdot(1+\varepsilon_n)  \leq  \frac{\KL(\rhohat \| \uprho_0)+ \ln \frac{N}{\delta}}{\hat \eta n}     + \frac{3 \hat\eta }{5 \alpha } \E_{h\sim \rhohat}	[\C_{\alpha}[\ell(h,X)]], \label{eq:goal2}
			\end{gather}
			with probability at least $1-2\delta$, where we recall that $\what Z=  \E_{h\sim \rhohat}[\ell(h,X)]$. Let $\hat \eta$ be an estimator which satisfies
			\begin{gather}
			\hat\eta \in   [ \eta_\star\wedge (\alpha/2), \ 2 \eta_\star] \cap \cG,  \label{eq:estimator}\quad 
			\text{where}  \quad  		\eta_\star \coloneqq \sqrt{\frac{5 \alpha\cdot(  \KL(\rhohat \| \uprho_0) +\ln \frac{N}{\delta})}{ 3 n \E_{h\sim \rhohat}	[\C_{\alpha}[\ell(h,X)]]}}
			\end{gather}
			is the unconstrained minimizer $\hat \eta$ of the RHS of \eqref{eq:goal2}. Since the loss $\ell$ has range in $[0,1]$, $\KL(\rhohat\| \uprho_0)\geq 0$, and $(\delta, n)\in]0,1/2[\times [2,+\infty[$, we have $\eta_\star\geq \sqrt{\alpha /n} \geq \min \cG$. This, with the fact that $\cG$ is in the form of a geometric progression with common ratio $2$ and $\max \cG=\alpha/2$, ensures the existence (and in fact the uniqueness) of $\hat\eta$ satisfying \eqref{eq:estimator}. 
		\paragraph{Case 1.} Suppose that $\eta_\star\leq \alpha/2$. In this case, the estimator $\hat\eta$ in \eqref{eq:estimator} satisfies $\eta_\star \leq  \hat\eta \leq 2 \eta_\star$. Plugging $\hat\eta$ into \eqref{eq:goal2} yields 
			\begin{align}
			\E_{h\sim \rhohat}	[\C_{\alpha}[\ell(h,X)]] - \what\C_{\alpha}[\what Z]  \leq    3\sqrt{\frac{3 \E_{h\sim \rhohat}	[\C_{\alpha}[\ell(h,X)]]\cdot  (\KL(\rhohat \| \uprho_0) +\ln \frac{N}{\delta})}{5 \alpha n}} +\what\C_{\alpha}[\what Z]  \cdot  \varepsilon_n.
			\end{align} 
			By applying Lemma~\ref{lem:conv} with $R = \E_{h\sim \rhohat}	[\C_{\alpha}[\ell(h,X)]]$,  $\what R  =\what\C_{\alpha}[\what Z]$, $A = \frac{27 (\KL(\rhohat \| \uprho_0) +\ln \frac{N}{\delta})}{5 \alpha n }$, and $B=\what\C_{\alpha}[\what Z] \cdot \varepsilon_n$,
			we get 
			\begin{align}
				\E_{h\sim \rhohat}	[\C_{\alpha}[\ell(h,X)]] - \what\C_{\alpha}[\what Z]& \leq    \sqrt{\frac{27 \what\C_{\alpha}[\what Z]\cdot  (\KL(\rhohat \| \uprho_0) +\ln \frac{N}{\delta})}{5 \alpha n}} + 2\what\C_{\alpha}[\what Z] \cdot \varepsilon_n\\  & \qquad  +  \frac{27(\KL(\rhohat \| \uprho_0) +\ln \frac{N}{\delta})}{5n\alpha }. \label{eq:middle}
			\end{align}
			\paragraph{Case 2.} Suppose now that $\eta_\star >\alpha/2$. In this case, $\hat\eta=\alpha/2$. Plugging this into \eqref{eq:goal2} and using the fact that $\eta_\star >\alpha/2$, yields
				\begin{align}
			\E_{h\sim \rhohat}	[\C_{\alpha}[\ell(h,X)]] - \what\C_{\alpha}[\what Z] \leq \frac{4(\KL(\rhohat \| \uprho_0)+ \ln \frac{N}{\delta})}{\alpha n} +  \what\C_{\alpha}[\what Z] \cdot \varepsilon_n.		\label{eq:lasty}	
				\end{align}
				Since $\what \C_{\alpha}[\what Z] \geq 0$ and $4 \leq 27/5$, the RHS of \eqref{eq:lasty} is less than the RHS of \eqref{eq:middle}, which completes the proof.
		\end{proof}
	
	\subsection{Proof of Lemma \ref{lem:lemmas}}
	\begin{proof}
		Suppose that $Z$ is $(\sigma,b)$-sub-exponential. Then, 
		\begin{align}
		\E[e^{\eta Z}] \leq e^{\frac{\eta^2 \sigma^2}{2}}, \quad \forall |\eta| \leq 1/b. \label{eq:inter0}
		\end{align}
		Using that $\E[q_\star\mid Z] \leq 1/\alpha$, and Cauchy-Schwartz, we get
		\begin{gather}
		\label{eq:inter}   
		|\eta Y| \leq  |\eta Z|/\alpha, \quad \forall \eta \in \reals, \\
		\shortintertext{and so, for all $|\eta|\leq \alpha/b$, we have}
		\E[e^{\eta Y}] \leq \E[e^{|\eta Y|}]\stackrel{\eqref{eq:inter}}{\leq} \E[e^{\frac{\eta Z}{\alpha}}] + \E[e^{-\frac{\eta Z}{\alpha}}] \stackrel{\eqref{eq:inter0}}{\leq} 2 e^{\frac{\eta^2 \sigma^2}{2 \alpha^2}}.  
		\end{gather}
		When $Z$ is  $\sigma$-sub-Gaussian case, the proof is the same, except that we replace $b$ by $0$.
	\end{proof}

			\subsection{Proof of Theorem \ref{thm:cvarconc}}
			\label{sec:proofconc}
				\begin{proof} Let $\cX=[0,1]$ and $f \equiv \op{id}$ be the identity map. By invoking Lemmas~\ref{lem:relation} and \ref{lem:conc} with $Z=f(X)=X$; and using \eqref{eq:ineqestimator} (which is a consequence of Lemma \ref{lem:newvar}), we get, for all $\eta \in[0,\alpha]$,
				\begin{align}
			\E_P \left[ \exp \left( n \eta \cdot \left( 	\C_{\alpha}[Z] - \what \C_{\alpha}[Z](1+\epsilon_n) - \frac{\eta \kappa({\eta/\alpha}) \C_{\alpha}[Z]}{\alpha} \right) \right) \right]\leq 1, \label{eq:intermediate}
				\end{align}
				with probability at least $1-\delta$, where $\epsilon_n$ is as in \eqref{eq:espilon}. By adding $ \C_{\alpha}[Z]\cdot \epsilon_n$ to both sides of \eqref{eq:intermediate} and using the fact that $\kappa(x)\leq 3/5$, for all $x\leq 1/2$, we get, for all $\eta\in[0,\alpha/2]$,
				\begin{align}
			\E_P \left[ \exp \left( n \eta \cdot \left( 	\C_{\alpha}[Z] - \what \C_{\alpha}[Z] - \left(\frac{3\eta  }{5 \alpha} +\epsilon_n\right) \C_{\alpha}[Z]   \right) \right) \right] \leq 1, \label{eq:midpoint}
				\end{align}
				with probability at least $1-\delta$. Let $W\coloneqq 	\C_{\alpha}[Z] - \what \C_{\alpha}[Z] - \left(\frac{3\eta  }{5 \alpha} +\epsilon_n\right) \C_{\alpha}[Z]$, and note that by \eqref{eq:midpoint}, we have \begin{align}P[\E_{P}[\exp ( n \eta  W )] \leq 1 ]\geq 1-\delta.
				\end{align}
				 Let $\mathcal{E}$ be the event that $\E_P[ \exp (n\eta W)] \leq  1$. With this, we have, for any $\delta \in(0,1)$ and all $\eta \in[0,\alpha/2]$,
				\begin{align}
				P\left[  	\C_{\alpha}[Z] - \what \C_{\alpha}[Z] \geq   \left(\frac{3\eta }{5 \alpha} +\epsilon_n\right) \C_{\alpha}[Z] + \frac{\ln \frac{1}{\delta}}{\eta n} \right]& = P\left[ e^{n \eta W} \geq \frac{1}{\delta}\right]\\ & = P\left[ \left.e^{n \eta W} \geq \frac{1}{\delta}  \right|  \mathcal{E} \right] \cdot P[\mathcal{E}] \\ & \quad  + P\left[ \left. e^{n \eta W} \geq \frac{1}{\delta}  \right|  \mathcal{E}^{\rm{c}} \right] \cdot (1-P[\mathcal{E}]),\\
				& \leq \delta \E[e^{n \eta W}\mid \mathcal{E}] + \delta, \label{eq:mid1}\\
				& \leq 2 \delta, \quad \quad (\text{by definition of $\mathcal{E}$}) \label{eq:mid2}
				\end{align}
				where \eqref{eq:mid1} follows by Markov's inequality and \eqref{eq:midpoint}. Now, we can re-express \eqref{eq:mid2} as 
				\begin{align}
								\C_{\alpha}[Z] - \what \C_{\alpha}[Z] \leq  \left(\frac{3\eta }{5 \alpha} +\epsilon_n\right) \C_{\alpha}[Z] + \frac{\ln \frac{1}{\delta}}{\eta n},\label{eq:reexpressed}
								\end{align}
				with probability at least $1-2\delta$. By setting $\eta= \sqrt{\frac{5 \alpha \ln \frac{1}{\delta}}{3 n\C_{\alpha}[Z]}}  \wedge \alpha/2$ (which does not depend on the samples), we get 
				\begin{align}
				\C_{\alpha}[Z] - \what \C_{\alpha}[Z] & \leq  \epsilon_n \C_{\alpha}[Z] +   \sqrt{ \frac{12  \C_{\alpha}[Z] \ln \frac{1}{\delta}}{5\alpha n}}\vee \frac{3\ln \frac{1}{\delta}}{\alpha n},
				\end{align}
				with probability at least $1-2\delta$. 
			\end{proof}

			\subsection{Proof of Theorem~\ref{thm:boundsunbounded}}
			\begin{proof}
				Let $\bar Z = Z-\E[Z]$. Suppose that $Z$ is $(\sigma,b)$-sub-exponential. In this case, by Lemma~\ref{lem:lemmas} the random variable $Y\coloneqq  \bar Z \cdot \E[q_\star \mid \bar Z]$ satisfies \eqref{eq:subexp}, and so by \citep[Theorem 2.19]{wainwright2019high}, we have 
				\begin{align}
				P\left[ \E[Y] - \frac{1}{n} \sum_{i=1}^n Y_i  \geq t  \right]  \leq \left\{  \begin{array}{cl} 2 e^{-\frac{n \alpha^2  t^2}{2 \sigma^2}}, & \text{if}\ 0 \leq  t\leq \frac{\sigma^2}{b \alpha};\\  2 e^{-\frac{n \alpha  t}{2 b}}  , & \text{if}\  t> \frac{\sigma^2}{b \alpha}.  \end{array}  \right. 	\label{eq:probY}
				\end{align}
				For any real random variables $A,B$, and $C$, we have $[A\geq C] \implies [A \geq B\ \text{or}\  B\geq C]$, and so $P[A\geq C] \leq P[A \geq B] +P[B\geq C]$. Applying this with $A =\C_{\alpha}[\bar Z]-  \what\C_{\alpha}[\bar Z] -|\bar Z_{(\ceil{n\alpha })}|  \epsilon_n$, $B=\E[Y]- \sum_{i=1}^n Y_i/n$, and $C=t\in \reals$. we get:
				\begin{align}
				P\left[\C_{\alpha}[\bar Z]-  \what\C_{\alpha}[\bar Z]- |\bar Z_{(\ceil{n\alpha })}|  \cdot \epsilon_n \geq t \right]& \leq P\left[\C_{\alpha}[\bar Z]-  \what\C_{\alpha}[\bar Z]-|\bar Z_{(\ceil{n\alpha })}|  \cdot \epsilon_n \geq \E[Y]- \frac{1}{n}\sum_{i=1}^n Y_i  \right] \\
				& \quad + P\left[ \E[Y]- \frac{1}{n}\sum_{i=1}^n Y_i \geq t \right], \\
				& \leq \delta +   \left\{  \begin{array}{cl} 2 e^{-\frac{n \alpha^2  t^2}{2 \sigma^2}}, & \text{if}\ 0 \leq  t\leq \frac{\sigma^2}{b \alpha};\\  2 e^{-\frac{n \alpha  t}{2 b}}  , & \text{if}\  t> \frac{\sigma^2}{b \alpha},  \end{array}  \right. \label{eq:prob}
				\end{align}
				where the last inequality follows by \eqref{eq:probY} and the fact that \eqref{eq:reduc} (with $Z$ replaced by $\bar Z$) holds with probability at least $1-\delta$. Since $\C_{\alpha}[Z]$ [resp. $\what\C_{\alpha}[Z]$] is a coherent risk measure, we have $\C_{\alpha}[\bar Z]= \C_{\alpha}[Z]-\E[Z]$ [resp. $\what\C_{\alpha}[\bar Z]= \what\C_{\alpha}[Z]-\E[Z]$], and so the LHS of \eqref{eq:prob} is equal to 
				\begin{align}
				P\left[\C_{\alpha}[Z]- \what\C_{\alpha}[Z] \geq t + |\bar Z_{(\ceil{n\alpha })}| \cdot \epsilon_n \right].
				\end{align}
				This with the fact that $\bar Z_{(\ceil{n\alpha })} = Z_{(\ceil{n\alpha })}  - \E[Z]$ completes the proof for the sub-exponential case. 
				
				When $Z$ is  $\sigma$-sub-Gaussian case, the proof is the same, except that we replace $b$ by $0$ and use the convention that $0/0=+\infty$.
			\end{proof}

\section{Beyond $\cvar$}
\label{app:gen}
First, we give a formal definition of a coherent risk measure (CRM):
\begin{definition}
	\label{def:coherent}
	We say that $\R\colon \mathcal{L}^1(\Omega) \rightarrow \reals \cup \{+\infty\}$ is a coherent risk measure if, for any $Z,Z' \in \mathcal{L}^1(\Omega)$ and $c\in \reals$, it satisfies the following axioms:
	(Positive Homogeneity) $\R[ \lambda Z] = \lambda \R[Z]$, for all $\lambda \in (0,1)$;
	(Monotonicity) $\R[Z] \leq \R[Z']$ if $Z\leq Z'$ a.s.;
	(Translation Equivariance) $\R[Z+c]= \R[Z]+ c$;
	(Sub-additivity) $\R[Z+ Z'] \leq \R[Z] + \R[Z']$. 
\end{definition} 

It is known that the conditional value at risk is a member of a class of CRMs called $\varphi$-entropic risk measures \cite{ahmadi2012entropic}. These CRMs are often used in the context of robust optimization \cite{namkoong2017}, and are perfect candidates to consider next in the context of this paper:
\begin{definition} Let $\varphi\colon [0,+\infty[\rightarrow\reals\cup \{+\infty\}$ be a closed convex function such that $\varphi(1)=0$. The $ \varphi$-entropic risk measure with divergence level $c$ is defined as 
	\begin{gather}
	\er^{c}_{\varphi}[Z]  \coloneqq \sup_{q \in \cQ_{\varphi}^{c}} \E_P[Zq], \ \  \text{where} \\ 
	\cQ_{\varphi}^c \coloneqq \left\{ q \in \mathcal{L}^1(\Omega) \left|\ \begin{matrix}  \exists Q \in \cM_P(\Omega)	,  q={\rmd  Q}/{\rmd P}, \\ \D_{\varphi} (Q \| P)  \leq c \end{matrix} \right.  \right\},  \label{eq:evar} 
	\end{gather}
	and $\D_{ \varphi}(\Q \| \P) \coloneqq \E_{\P}[ \varphi (q)]$ is the $\varphi$-divergence between two distributions $\Q$ and $\P$, where $Q\ll P$ and $q=\frac{\rmd Q}{\rmd P}$.
\end{definition}
As mentioned above, $\cvar_{\alpha}[Z]$ is a $\varphi$-entropic risk measure; in fact, it is the $\varphi$-entropic risk measure at level $c =0$ with $\varphi(\cdot) \coloneqq\iota_{ [0,1/\alpha]}(\cdot)$, where for a set $\mathcal{C} \subseteq \reals$, $\iota_{\mathcal{C}}(x)=0$ if $x\in \mathcal{C}$; and $+\infty$ otherwise \cite{ahmadi2012entropic}. 

The natural estimator $\what\er_{\varphi}^c[Z]$ of $\er_{\varphi}^c[Z]$ is defined by \cite{ahmadi2012entropic}
\begin{align}
\hspace{-0.1cm}\what\er_\varphi^c[Z] = \inf_{\nu >0, \mu \in \reals} \left\{ \mu + \nu \E_{\what P_n}\left[\varphi^\star\left(\frac{Z-\mu}{\nu} -c\right)\right] \right\}.
\end{align}
Extending the results of Lemmas \ref{lem:keylem} and \ref{lem:relation} comes down to finding an auxiliary estimator $\wtilde\er_{\varphi}^c[Z]$ of $\er_{\varphi}^c[Z]$ which satisfies (as in Lemma \ref{lem:relation}) $\wtilde\er_{\varphi}^c[Z] \leq \what\er_{\varphi}^c[Z]\cdot (1+\epsilon_n)$, for some ``small'' $\epsilon_n$, and 
\begin{align}
\frac{1}{n}\sum_{i=1}^n Z_i \cdot \E[q_\star \mid Z_i] \leq \wtilde\er_{\varphi}^c[Z],
\end{align} 
with high probability, where $q_\star \in \argmin_{q\in \cQ^c_{\varphi}} \E[Z q]$. The similarities between the expressions of $\what\er_{\varphi}^c[Z]$ and $\what\C_{\alpha}[Z]$ hint that it might be possible to find such an estimator by carefully constructing a set $\wtilde\cQ^c_{\varphi}$ to play the role of the $\wtilde\cQ_{\alpha}$ in Section \ref{sec:PACBayes}. We leave such investigations for future work.

\end{document}